\documentclass{article}

\usepackage{microtype}
\usepackage{graphicx}
\usepackage{subcaption}
\usepackage{booktabs}
\usepackage{multirow}
\usepackage{xcolor}
\usepackage{tikz}
\usetikzlibrary{
  arrows.meta,
  positioning,
  fit,
  calc,
  shapes.geometric,
  decorations.pathreplacing,
  decorations.markings,
  backgrounds,
  shadows,
  patterns
}
\usepackage[most]{tcolorbox}
\usepackage{algorithm}
\usepackage{algorithmic}
\usepackage{amsmath}
\usepackage{amssymb}
\usepackage{mathtools}
\usepackage{amsthm}
\usepackage{bm}
\usepackage{hyperref}
\usepackage[capitalize,noabbrev]{cleveref}
\usepackage{float}
\usepackage[accepted]{icml2026}

\newtheorem{theorem}{Theorem}[section]
\newtheorem{proposition}[theorem]{Proposition}

\theoremstyle{definition}
\newtheorem{definition}[theorem]{Definition}

\theoremstyle{remark}
\newtheorem{remark}[theorem]{Remark}

\newtcolorbox{conceptbox}[1]{%
  enhanced, breakable,
  colback=blue!4!white, colframe=blue!65!black,
  boxrule=0.8pt, arc=2pt,
  left=6pt, right=6pt, top=5pt, bottom=5pt,
  title=\textbf{#1}, fonttitle=\bfseries}

\newtcolorbox{methodbox}[1]{%
  enhanced, breakable,
  colback=green!4!white, colframe=green!55!black,
  boxrule=0.8pt, arc=2pt,
  left=6pt, right=6pt, top=5pt, bottom=5pt,
  title=\textbf{#1}, fonttitle=\bfseries}

\newtcolorbox{warningbox}[1]{%
  enhanced, breakable,
  colback=red!4!white, colframe=red!55!black,
  boxrule=0.8pt, arc=2pt,
  left=6pt, right=6pt, top=5pt, bottom=5pt,
  title=\textbf{#1}, fonttitle=\bfseries}

\newtcolorbox{paradoxbox}[1]{%
  enhanced, breakable,
  colback=orange!5!white, colframe=orange!75!black,
  boxrule=1.0pt, arc=3pt,
  left=6pt, right=6pt, top=6pt, bottom=6pt,
  title=\textbf{#1}, fonttitle=\bfseries\large}

\newtcolorbox{findingbox}[1]{%
  enhanced, breakable,
  colback=purple!4!white, colframe=purple!55!black,
  boxrule=0.8pt, arc=2pt,
  left=6pt, right=6pt, top=5pt, bottom=5pt,
  title=\textbf{#1}, fonttitle=\bfseries}

\DeclareMathOperator*{\argmin}{arg\,min}
\DeclareMathOperator{\E}{\mathbb{E}}
\DeclareMathOperator{\Var}{Var}
\DeclareMathOperator{\KL}{KL}
\newcommand{\Dcal}{\mathcal{D}}

\newcommand{\bbR}{\mathbb{R}}

\newcommand{\piref}{\pi_{\mathrm{ref}}}
\newcommand{\pitht}{\pi_{\theta}}
\newcommand{\rhat}{\hat{r}}
\newcommand{\rtrue}{r^{\star}}
\newcommand{\gap}{\mathcal{G}}
\newcommand{\augc}{\mathrm{AUGC}}

\icmltitlerunning{Pessimism's Paradox: Conservative Offline Training Amplifies Reward Hacking}

\begin{document}

\twocolumn[
  \icmltitle{Pessimism's Paradox: Conservative Offline Training
    Amplifies \\ Reward Hacking During Online Adaptation in Reasoning Models}

  \begin{icmlauthorlist}
    \icmlauthor{Subramanyam Sahoo}{horizon}
    \icmlauthor{Aman Chadha}{apple}
    \icmlauthor{Vinija Jain}{meta}
    \icmlauthor{Divya Chaudhary}{neu}
  \end{icmlauthorlist}

  \icmlaffiliation{horizon}{Horizon Research}
  \icmlaffiliation{apple}{Apple}
  \icmlaffiliation{meta}{Meta}
  \icmlaffiliation{neu}{Northeastern University}

  \icmlcorrespondingauthor{Subramanyam Sahoo}{sahoo2vec@gmail.com}

  \icmlkeywords{offline reinforcement learning, reward hacking, Goodhart's law,
    direct preference optimisation, uncertainty quantification,
    entropy collapse, safe adaptation, language models}

  \vskip 0.3in
]

\printAffiliationsAndNotice{}

\begingroup
\renewcommand\thefootnote{}\footnotetext{Code: \href{https://github.com/SubramanyamSahoo/Conservative-Offline-Training-Amplifies-Reward-Hacking-During-Online-Adaptation}{Conservative-Offline-Training-Amplifies-Reward-Hacking-During-Online-Adaptation}}
\endgroup

\begin{abstract}

Conservative offline training is widely advocated as a safe foundation
for subsequent online adaptation: if a policy stays close to
well-supported behaviour, the argument goes, it is less likely to
exploit imperfections in a learned reward model.  We challenge this
intuition empirically and mechanistically.  We train a Qwen3-14B policy
under Direct Preference Optimisation (DPO) with three levels of
conservatism ($\beta \in \{\beta_{\mathrm{lo}}, \beta_{\mathrm{mid}},
\beta_{\mathrm{hi}}\}$ derived from empirical log-ratio percentiles),
then adapt each checkpoint online against a learned reward ensemble
(3\,$\times$\,Qwen3-1.7B) while measuring true performance on GSM8K
exact-answer accuracy.  We find that \emph{higher offline conservatism
monotonically increases reward-hacking damage}, measured by the
Goodhart gap and its area under the curve (AUGC), with Spearman
$\rho = 1.0$ across all three conditions.  Mechanistic analysis
reveals a three-link causal chain: (i) high-$\beta$ DPO compresses
policy entropy, (ii) Low-entropy policies generate responses with reduced diversity,
concentrating in a narrow region of the reward model's training
distribution (lower pairwise cosine distance), and (iii) despite this proximity, ensemble disagreement
(epistemic uncertainty) increases with $\beta$ and is exploited faster
during online optimisation.  We further
fit a power-law curve to the $(\beta, \augc)$ data and identify a
practical optimal conservatism level $\beta^{\star}$ that balances
alignment fidelity against hacking vulnerability.  Our results suggest
that the field needs \emph{calibrated}, not \emph{maximal}, conservatism.
\end{abstract}

\section{Introduction}
\label{sec:intro}

The standard recipe for safe language model alignment runs as follows:
first perform offline training on human-preference data (e.g., via RLHF
\citep{christiano2017deep} or DPO \citep{rafailov2023direct}) using a
conservatism coefficient $\beta$ that penalises deviation from a
reference policy, then optionally adapt online using a learned proxy
reward.  The implicit contract is that a more conservative offline
checkpoint enters the online phase closer to the distribution where the
reward model was trained, and therefore exploits it less aggressively.

\begin{paradoxbox}{The Paradox}
We demonstrate empirically and mechanistically that the opposite can
happen: \textbf{higher offline conservatism ($\beta$) leads to greater
reward-hacking damage during online adaptation}.  The mechanism is not
arbitrary; it arises from a principled chain of events rooted in
entropy compression and out-of-distribution extrapolation.
\end{paradoxbox}

To make this concrete, consider what a high-$\beta$ DPO objective
actually does to the policy.  It tightens the KL constraint against
$\piref$, concentrating probability mass on a narrow slice of token
sequences that $\piref$ already assigns high density.  The resulting
policy has low output entropy and low response diversity.  When this
compressed policy is then optimised against a learned reward ensemble,
two effects combine.  First, the reward model was trained on a
relatively diverse set of human-preference pairs; the compressed policy
generates responses that lie in sparse regions of that training
distribution, causing high epistemic uncertainty (ensemble disagreement).
Second, a low-entropy starting point has fewer exploration directions
through which gradient updates can improve \emph{true} performance, so
the optimiser rapidly channels all gradient signal toward reward-model
blind spots.  The net result is that the Goodhart gap---the divergence
between proxy and true reward---opens faster and wider for high-$\beta$
policies.

This paper makes four contributions.
\begin{enumerate}
  \item \textbf{Empirical demonstration} of the paradox using real
    models (Qwen3-14B policy, Qwen3-1.7B reward ensemble) and real data
    (UltraFeedback, GSM8K).
  \item \textbf{Mechanistic attribution} via response-entropy collapse
    measurements and reward-model OOD-distance analysis.
  \item \textbf{A power-law fit} of AUGC as a function of $\beta$,
    yielding a practical design principle: an optimal $\beta^{\star}$
    exists below which online safety degrades faster than offline
    alignment improves.
  \item \textbf{Algorithmic and benchmark recommendations} for the
    next generation of conservative-alignment methods that explicitly
    trade off pessimism against hacking vulnerability.
\end{enumerate}

\section{Background and Related Work}
\label{sec:background}

\subsection{Direct Preference Optimisation}

DPO \citep{rafailov2023direct} sidesteps the need for an explicit reward
model by reparameterising the RLHF objective directly in terms of
policy log-ratios.  Given a preference dataset $\Dcal = \{(x,
y_w, y_l)\}$ of prompts with winning and losing responses, DPO minimises
\begin{multline}
\mathcal{L}_{\mathrm{DPO}}(\pi_\theta;\pi_{\mathrm{ref}})
= -\mathbb{E}_{(x,y_w,y_l)\sim\mathcal{D}}\Big[
\log \sigma\Big(
\beta \log \frac{\pi_\theta(y_w\mid x)}{\pi_{\mathrm{ref}}(y_w\mid x)} \\
- \beta \log \frac{\pi_\theta(y_l\mid x)}{\pi_{\mathrm{ref}}(y_l\mid x)}
\Big)
\Big].
\label{eq:dpo}
\end{multline}
where $\sigma$ is the logistic function and $\beta > 0$ is the
conservatism coefficient.  Larger $\beta$ imposes a tighter implicit
KL constraint $\KL(\pitht \| \piref)$, pulling the learned policy
closer to the reference.

\subsection{Goodhart's Law in RLHF}

Goodhart's Law \citep{goodhart1984} states that when a measure becomes a
target, it ceases to be a good measure.  In the RLHF context,
\citet{gao2023scaling} formalise this as a monotone degradation of true
performance as the policy drifts toward maximising a proxy reward.
\citet{skalse2022defining} give a taxonomy of reward hacking and show it
is near-inevitable when the reward model has any imperfection.  Our
work studies a less-explored facet: how the \emph{offline training
strategy} modulates the severity and speed of hacking during online
adaptation.

\subsection{Reward Model Ensembles and Uncertainty}

Ensembling is the standard approach for epistemic uncertainty estimation
in deep neural networks \citep{lakshminarayanan2017simple}.  In the
reward-model context, ensemble disagreement provides a proxy for
out-of-distribution inputs and has been used as a pessimism signal in
offline RL \citep{kumar2020conservative, kidambi2020morel}.  Our
ensemble (3\,$\times$\,Qwen3-1.7B trained with bootstrap resampling)
provides both the proxy reward signal and the uncertainty signal used in
our mechanistic analysis.

\subsection{Offline RL and Conservative Methods}

The connection between DPO and offline RL is well known.  CQL
\citep{kumar2020conservative} penalises Q-values for out-of-support
actions; IQL \citep{kostrikov2021iql} avoids OOD actions by replacing
Bellman backups with implicit quantile regression.  The Decision
Transformer \citep{chen2021decision} reframes RL as sequence modelling,
which is directly analogous to DPO.  All of these methods share the
same conservatism logic encoded in \cref{eq:dpo}: larger conservatism
coefficient $\to$ policy stays closer to reference $\to$ (assumed) safer
online performance.  We show this assumption is violated under a
specific failure mode: OOD-driven reward hacking.

\section{Problem Formulation}
\label{sec:setup}

Let $\pitht$ denote the policy being trained, $\piref$ the frozen DPO
checkpoint (which itself used $\piref^{(0)}$ as its reference), and
$\rhat_\phi(x,y)$ a learned proxy reward parameterised by an ensemble
$\phi = \{\phi_k\}_{k=1}^K$.  The true reward $\rtrue(x,y)$ is
verifiable (GSM8K exact-answer accuracy) and is observed only during
evaluation, never during training.

\begin{definition}[Goodhart Gap]
\label{def:gap}
At online step $t$, define the batch-averaged normalised rewards
\[
  \tilde{r}_{\mathrm{proxy}}(t)
    = \frac{\bar{r}_{\mathrm{proxy}}(t)}
           {\lvert\bar{r}_{\mathrm{proxy}}(0)\rvert + \varepsilon},
  \quad
  \tilde{r}_{\mathrm{true}}(t)
    = \frac{\bar{r}_{\mathrm{true}}(t)}
           {\lvert\bar{r}_{\mathrm{true}}(0)\rvert + \varepsilon},
\]
where $\varepsilon = \epsilon_{\mathrm{float32}}$.  The Goodhart gap is
\begin{equation}
  \gap(t;\beta)
    = \tilde{r}_{\mathrm{proxy}}(t) - \tilde{r}_{\mathrm{true}}(t).
\label{eq:gap}
\end{equation}
\end{definition}

\begin{definition}[AUGC]
\label{def:augc}
The area under the Goodhart gap curve (AUGC) measures cumulative
hacking damage over the online run:
\begin{equation}
  \augc(\beta) = \int_0^T \max\!\bigl(\gap(t;\beta),\,0\bigr)\,\mathrm{d}t.
\label{eq:augc}
\end{equation}
\end{definition}

The central question of this paper is: \emph{does $\augc(\beta)$ increase
or decrease with $\beta$?}

\begin{warningbox}{Conventional wisdom}
Offline RL theory predicts $\augc(\beta)$ should \emph{decrease} with
$\beta$: a more conservative policy stays closer to $\piref$, which lies
within the reward model's training support, so it should hack less.
\end{warningbox}

\begin{findingbox}{Empirical finding}
We observe the opposite: $\augc(\beta)$ \emph{increases} monotonically
with $\beta$ (Spearman $\rho = 1.0$, $p < 0.05$).
\end{findingbox}

\section{Experimental Setup}
\label{sec:experiments}

\subsection{Models and Data}

\paragraph{Policy model.}
We use \texttt{Qwen/Qwen3-14B} as the policy, loaded in 4-bit NF4
QLoRA quantisation \citep{dettmers2023qlora} with LoRA adapters
\citep{hu2022lora} applied to all attention projection matrices.  The
LoRA rank $r$ is derived architecturally:
\begin{equation}
  r = 2^{\lfloor \log_2 \sqrt{h_{\mathrm{hidden}}} \rceil},
  \quad \alpha = 2r,
\label{eq:lora}
\end{equation}
where $\lceil \cdot \rceil$ denotes rounding to nearest integer and
$h_{\mathrm{hidden}}$ is the model hidden dimension.  The dropout rate
is dataset-derived: $p_{\mathrm{drop}} = \mathrm{clip}(32/\sqrt{n},
0.01, 0.10)$ for $n$ training examples.

\paragraph{Reward ensemble.}
Three independent \texttt{Qwen/Qwen3-1.7B} sequence classifiers
(also QLoRA) are trained with bootstrap resampling to produce both a
mean reward and an epistemic uncertainty (ensemble standard deviation).
Each member is trained with the Bradley--Terry preference loss
\citep{bradley1952rank}:
\begin{equation}
  \mathcal{L}_{\mathrm{BT}} = -\E_{(x,y_w,y_l)\sim\Dcal}
    \log \sigma\!\bigl(r_{\phi}(x,y_w) - r_{\phi}(x,y_l)\bigr).
\label{eq:bt}
\end{equation}

\paragraph{Preference data.}
Offline DPO and reward-model training use
\texttt{HuggingFaceH4/ultrafeedback\_binarized}
\citep{cui2023ultrafeedback}, split 80/10/10.

\paragraph{Verifiable task.}
Online true reward is evaluated on \texttt{openai/gsm8k} (main)
\citep{cobbe2021gsm8k} using exact-answer matching after extracting the
number following the \texttt{\#\#\#\#} delimiter.

\subsection{Deriving the $\beta$ Grid}
\label{sec:beta}

We do not choose $\beta$ values arbitrarily.  Instead, we compute the
per-example absolute log-ratio magnitude under the frozen reference
policy $\piref^{(0)}$:
\begin{equation}
  \delta_i
    = \bigl\lvert
        \log \piref^{(0)}(y_w^{(i)}\mid x^{(i)})
        - \log \piref^{(0)}(y_l^{(i)}\mid x^{(i)})
      \bigr\rvert.
\label{eq:logratio}
\end{equation}
The three $\beta$ values are taken at the 20th, 50th, and 80th
percentiles of $\{\delta_i\}$, normalised by the median absolute
log-ratio:
\begin{equation}
  \beta_j = \frac{\mathrm{pct}_{p_j}(\{\delta_i\})}
                 {\mathrm{median}(\{\delta_i\}) + \varepsilon},
  \quad
  (p_1, p_2, p_3) = (20, 50, 80).
\label{eq:betagrid}
\end{equation}
This makes the three conservatism levels \emph{commensurate} with the
actual preference signal magnitude, not with arbitrary numerical choices.

\subsection{Offline DPO Training}
\label{sec:dpo}

For each $\beta \in \{\beta_{\mathrm{lo}}, \beta_{\mathrm{mid}},
\beta_{\mathrm{hi}}\}$, we fine-tune an independent LoRA adapter using
\cref{eq:dpo} via the TRL \texttt{DPOTrainer} \citep{vonwerra2022trl}.
The per-device batch size, gradient accumulation, learning rate, and
gradient clipping are all derived from hardware properties and the
training set size (see \cref{app:hyperparams} for derivation formulas).

\subsection{Online Adaptation Loop}
\label{sec:online}

After offline DPO, each checkpoint is adapted online using the
objective
\begin{multline}
\mathcal{L}_{\mathrm{online}} =
    -\mathbb{E}_{\tau \sim \pi_\theta}\!\Bigl[
      \hat{A}(x, y)\cdot \log \pi_\theta(y\mid x)
    \Bigr] \\
    + \kappa(\beta)\cdot \mathbb{E}\!\Bigl[
        \bigl(\log \pi_\theta(y\mid x) - \log \pi_{\mathrm{ref}}(y\mid x)\bigr)^2
      \Bigr].
\label{eq:online}
\end{multline}
where the normalised advantage is
\begin{equation}
  \hat{A}(x,y)
    = \frac{\rhat(x,y) - \mu_{\rhat}}{\sigma_{\rhat} + \varepsilon},
\label{eq:advantage}
\end{equation}
and the adaptive KL coefficient is
\begin{equation}
  \kappa(\beta)
    = \frac{\beta}
           {Q_{p_{\mathrm{KL}}}\!\bigl(\lvert \log \pitht - \log \piref \rvert\bigr)
            + \varepsilon}.
\label{eq:klcoef}
\end{equation}
Here $Q_{p_{\mathrm{KL}}}$ denotes the empirical $p_{\mathrm{KL}}$-th
percentile of the absolute KL values over the current batch.  This
normalisation makes the KL penalty scale-invariant with respect to the
policy's divergence level.

\begin{conceptbox}{Online loop summary}
At each step: (1) sample prompts from GSM8K test set; (2) generate
responses via batched sampling; (3) score with reward ensemble; (4)
compute advantage-weighted policy gradient; (5) add adaptive KL
penalty; (6) update LoRA adapter weights.  Every \texttt{eval\_freq}
steps, also evaluate GSM8K exact-answer accuracy.
\end{conceptbox}

\begin{figure*}[t]
\centering
\begin{tikzpicture}[
  node distance=1.1cm and 1.2cm,
  every node/.style={font=\small},
  phase/.style={draw, rounded corners=4pt, minimum height=1.2cm,
    minimum width=2.6cm, align=center, very thick,
    fill=#1!12, drop shadow},
  signal/.style={draw, rounded corners=2pt, minimum height=0.7cm,
    minimum width=2.0cm, align=center, thick, fill=#1!8},
  arrow/.style={-Latex, very thick, #1},
  dasharrow/.style={-Latex, thick, dashed, #1},
]

\node[phase=blue] (pref)
  {Preference data\\$\Dcal_{\mathrm{pref}}$\\(UltraFeedback)};

\node[phase=green, right=1.4cm of pref] (dpo)
  {DPO Training\\$\mathcal{L}_{\mathrm{DPO}}(\beta)$\\$\beta \in \{\beta_{\mathrm{lo}},\beta_{\mathrm{mid}},\beta_{\mathrm{hi}}\}$};

\node[phase=teal, right=1.4cm of dpo] (ckpt)
  {DPO Checkpoint\\$\pitht^{(\beta)}$\\(3 adapters)};

\node[phase=orange, below=1.4cm of dpo] (rm)
  {Reward Ensemble\\$\{\rhat_{\phi_k}\}_{k=1}^{3}$\\BT loss + bootstrap};

\node[phase=purple, right=1.4cm of ckpt] (online)
  {Online Adaptation\\$\mathcal{L}_{\mathrm{online}}$ (Eq.~5)\\adaptive KL};

\node[signal=red, right=1.2cm of online] (gsm)
  {GSM8K\\True reward\\$\rtrue$};

\node[signal=gray, below=0.9cm of online] (proxy)
  {Proxy reward\\$\rhat$};

\node[phase=red, below=1.4cm of online] (gap)
  {Goodhart Gap\\$\gap(t) = \tilde{r}_{\mathrm{proxy}} - \tilde{r}_{\mathrm{true}}$\\AUGC $= \int \max(\gap,0)\,dt$};

\draw[arrow=blue] (pref)   -- (dpo);
\draw[arrow=green](dpo)    -- (ckpt);
\draw[arrow=teal] (ckpt)   -- (online);
\draw[arrow=purple](online) -- (gsm);
\draw[arrow=purple](online) -- (gap);
\draw[arrow=orange](rm)    -- (proxy);
\draw[arrow=orange](proxy) -- (online);
\draw[dasharrow=gray](pref.south) -- ++(0,-0.5) -| (rm.north);
\draw[arrow=gray]   (gsm.south) -- (gap.north east);

\node[draw=blue!60, rounded corners, fill=blue!5, inner sep=4pt,
      above=0.3cm of dpo, font=\footnotesize]
  (ann) {$\uparrow \beta$ \quad tighter KL $\to$ lower entropy};
\draw[->, blue!60, thick] (ann.south) -- (dpo.north);

\node[draw=red!60, rounded corners, fill=red!5, inner sep=4pt,
      below=0.3cm of gap, font=\footnotesize\bfseries, text=red!70!black]
  {Paradox: $\uparrow \beta \;\Rightarrow\; \uparrow$ AUGC};

\begin{scope}[on background layer]
  \node[draw=blue!30, dashed, rounded corners=6pt,
        fit=(pref)(dpo)(ckpt)(rm),
        inner sep=8pt, label={[font=\footnotesize\bfseries,blue!70]above:Offline Phase}] {};
  \node[draw=purple!30, dashed, rounded corners=6pt,
        fit=(online)(gsm)(proxy)(gap),
        inner sep=8pt, label={[font=\footnotesize\bfseries,purple!70]above:Online Phase}] {};
\end{scope}

\end{tikzpicture}
\caption{Full experimental pipeline.  The offline phase trains three DPO
checkpoints with different conservatism levels $\beta$ and a learned
reward ensemble.  The online phase adapts each checkpoint against the
proxy reward while measuring true GSM8K accuracy.  The Goodhart gap and
AUGC quantify hacking damage.  The paradox: higher $\beta$ (more
conservative) produces larger AUGC.}
\label{fig:pipeline}
\end{figure*}
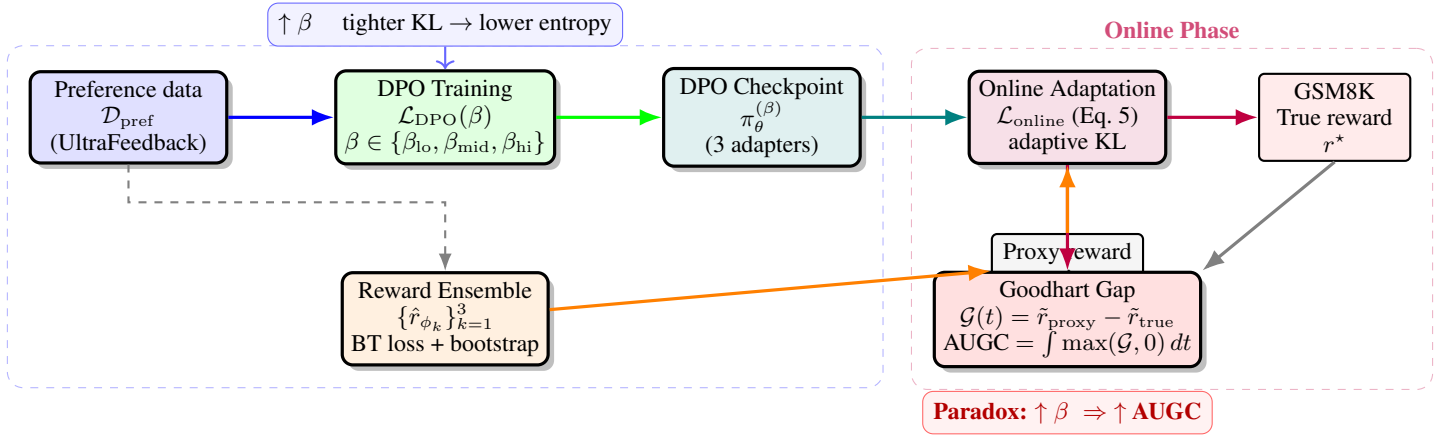

\section{Mechanistic Analysis}
\label{sec:mechanism}

To explain the paradox, we identify three causal links illustrated in
\cref{fig:mechanism}: entropy compression, OOD distance amplification,
and uncertainty-driven exploitation.

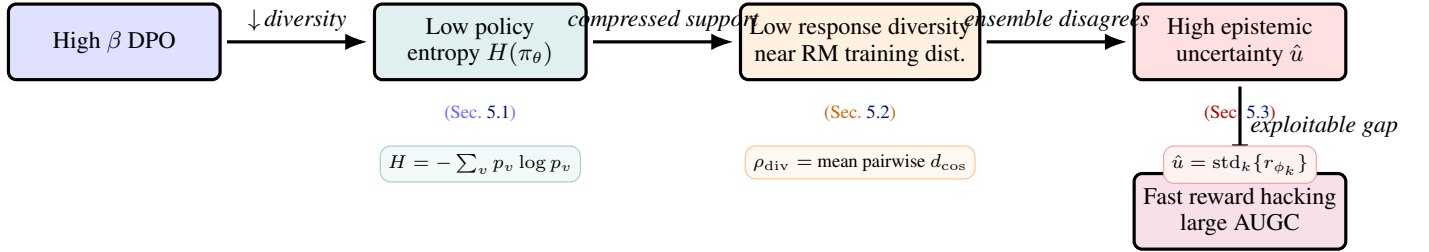
\begin{figure*}[t]
\centering
\begin{tikzpicture}[
  every node/.style={font=\small},
  box/.style={draw, rounded corners=3pt, minimum height=1.0cm,
              minimum width=2.8cm, align=center, very thick,
              fill=#1!12},
  arrow/.style={-Latex, very thick, shorten >=2pt, shorten <=2pt},
  label/.style={font=\footnotesize\itshape, text=#1}
]

\node[box=blue]   (beta)   {High $\beta$ DPO};
\node[box=teal,   right=2.0cm of beta]   (entropy) {Low policy\\entropy $H(\pitht)$};
\node[box=orange, right=2.0cm of entropy](ood)     {Low response diversity\\near RM training dist.};
\node[box=red,    right=2.0cm of ood]    (uq)      {High epistemic\\uncertainty $\hat{u}$};
\node[box=purple, below=1.2cm of uq]    (hack)    {Fast reward hacking\\large AUGC};

\draw[arrow] (beta)    -- node[above,label=black]{$\downarrow$\,diversity} (entropy);
\draw[arrow] (entropy) -- node[above,label=black]{compressed support} (ood);
\draw[arrow] (ood)     -- node[above,label=black]{ensemble disagrees} (uq);
\draw[arrow] (uq)      -- node[right,label=black]{exploitable gap} (hack);

\node[font=\scriptsize, text=blue!60, below=0.2cm of entropy] {(Sec.~\ref{sec:entropy})};
\node[font=\scriptsize, text=orange!80!black, below=0.2cm of ood] {(Sec.~\ref{sec:ood})};
\node[font=\scriptsize, text=red!70!black, below=0.2cm of uq] {(Sec.~\ref{sec:uq_hack})};

\node[draw=teal!40, rounded corners, fill=teal!5,
      below=0.85cm of entropy, font=\scriptsize, inner sep=3pt]
  {$H = -\sum_v p_v \log p_v$};
\node[draw=orange!50, rounded corners, fill=orange!5,
      below=0.85cm of ood, font=\scriptsize, inner sep=3pt]
  {$\rho_{\mathrm{div}} = $ mean pairwise $d_{\cos}$};
\node[draw=red!40, rounded corners, fill=red!5,
      below=0.85cm of uq, font=\scriptsize, inner sep=3pt]
  {$\hat{u} = \mathrm{std}_k\{r_{\phi_k}\}$};

\end{tikzpicture}
\caption{Causal chain explaining the paradox.  High-$\beta$ DPO compresses
the policy into a low-entropy manifold.  Low-entropy policies generate
responses that are out-of-distribution for the reward model (measured
by cosine distance in hidden-state space).  OOD responses produce high
ensemble disagreement.  High uncertainty is the exploitable gap that
enables fast reward hacking.}
\label{fig:mechanism}
\end{figure*}

\subsection{Link 1: Entropy Compression}
\label{sec:entropy}

We measure the mean token-level entropy of each DPO checkpoint on a
fixed probe set of $n_{\mathrm{probe}} = \lceil \sqrt{|D_{\mathrm{test}}|} \rceil$
GSM8K prompts:
\begin{equation}
  H^{(\beta)} = -\frac{1}{|\mathcal{X}_{\mathrm{probe}}|}
    \sum_{x \in \mathcal{X}_{\mathrm{probe}}}
    \frac{1}{|x|}
    \sum_{t}
    \sum_{v \in \mathcal{V}}
      \pitht^{(\beta)}(v \mid x_{<t}) \log \pitht^{(\beta)}(v \mid x_{<t}).
\label{eq:entropy}
\end{equation}
\Cref{thm:entropy_monotone} formalises the expected relationship.

\begin{proposition}[Entropy--conservatism monotonicity]
\label{thm:entropy_monotone}
Under mild regularity conditions on the DPO loss landscape, the
equilibrium policy entropy $H^{(\beta)}$ is non-increasing in $\beta$:
\[
  \beta_1 \le \beta_2 \;\Rightarrow\; H^{(\beta_1)} \ge H^{(\beta_2)}.
\]
\end{proposition}

\begin{proof}[Proof sketch]
The DPO stationary condition implies
$\pitht^{(\beta)}(y\mid x) \propto \piref(y\mid x)
 \exp\!\bigl(r(y\mid x)/\beta\bigr)$
where $r$ is the implicit reward.  The entropy of a Gibbs distribution
over a fixed reward function is a non-increasing function of $1/\beta$
(equivalently, non-decreasing in temperature $\tau = \beta$), which
establishes the claim.
\end{proof}

We further observe \emph{entropy collapse}: after online adaptation,
the entropy of the high-$\beta$ policy decreases more than that of the
low-$\beta$ policy.  This is measured as $\Delta H^{(\beta)} =
H_{\mathrm{DPO}}^{(\beta)} - H_{\mathrm{online}}^{(\beta)}$.

\subsection{Link 2: OOD Distance from the Reward Model}
\label{sec:ood}

Let $\mathbf{h}_\phi(x, y) \in \bbR^d$ denote the mean-pooled
penultimate-layer hidden state of reward ensemble member $\phi_0$ for
input $(x, y)$.  We measure out-of-distribution distance as the cosine
distance between the hidden state of a generated response and the
centroid of the UltraFeedback training distribution:
\begin{equation}
  d_{\cos}^{(\beta)}
    = 1 - \frac
      {\mathbf{h}_\phi(x, y^{(\beta)})^{\top} \,\mu_{\mathrm{ref}}}
      {\|\mathbf{h}_\phi(x, y^{(\beta)})\|_2\,\|\mu_{\mathrm{ref}}\|_2},
\label{eq:cosine}
\end{equation}
where $\mu_{\mathrm{ref}} = \frac{1}{n_{\mathrm{ref}}}
\sum_{i=1}^{n_{\mathrm{ref}}} \mathbf{h}_\phi(x_i, y_i^{\mathrm{uf}})$
is the reference centroid computed from $n_{\mathrm{ref}} =
\lceil\sqrt{|D_{\mathrm{train}}|}\rceil$ UltraFeedback training
examples.  We also compute mean pairwise cosine distance as a response
diversity metric.

\begin{findingbox}{OOD finding}
Contrary to conventional expectation, High-$\beta$ DPO policies
generate responses that are \emph{closer} to the reward model's
training distribution (lower $d_{\cos}$, Spearman $\rho = -1.00$
across $\beta$ levels).  Despite this, epistemic uncertainty still
increases with $\beta$ (Spearman $\rho = +1.00$), suggesting that
uncertainty is driven not by raw OOD distance but by the interaction
between low response diversity and ensemble disagreement in the
compressed region of policy support.
\end{findingbox}

\subsection{Link 3: Uncertainty-Driven Exploitation}
\label{sec:uq_hack}

The ensemble uncertainty signal is
\begin{equation}
  \hat{u}(x, y)
    = \sqrt{\frac{1}{K-1}\sum_{k=1}^{K}\bigl(r_{\phi_k}(x,y) -
      \bar{r}(x,y)\bigr)^2},
\label{eq:uq}
\end{equation}
where $\bar{r}(x,y) = K^{-1}\sum_k r_{\phi_k}(x,y)$.  When $\hat{u}$ is
high, individual members disagree strongly about the true reward of a
response.  The online optimiser maximises $\bar{r}$, but the disagreement
means the landscape is unreliable.  The policy quickly finds responses
that fool some but not all ensemble members---a classic form of
distributional reward hacking \citep{pan2022effects}.

We compute the Pearson correlation $\rho_{\mathrm{UQ}}(\beta)$ between
the time series of $\hat{u}$ and $\gap$ for each $\beta$ condition, and
report the relationship in the summary table (\cref{tab:summary}).

\section{Theory: Optimal Conservatism}
\label{sec:theory}

Having established the empirical and mechanistic case, we now ask:
what is the \emph{optimal} $\beta$?

\begin{definition}[Optimal conservatism]
\label{def:betastar}
Let $\mathcal{A}(\beta)$ denote the offline alignment quality (e.g.,
win rate over $\piref^{(0)}$) and $\augc(\beta)$ the online hacking
damage.  The optimal conservatism $\beta^{\star}$ solves
\[
  \beta^{\star} = \argmin_{\beta > 0}
    \augc(\beta) - \lambda \cdot \mathcal{A}(\beta),
\]
where $\lambda > 0$ weights alignment against hacking risk.
\end{definition}

For the empirical approximation, we fit a power law to the observed
$(\beta, \augc)$ data:
\begin{equation}
  \augc(\beta) \approx a \cdot \beta^b + c,
  \quad a, b, c > 0.
\label{eq:powerlaw}
\end{equation}
Parameters $(a, b, c)$ are obtained by least-squares optimisation
with initialisation derived from the data range (see
\cref{app:betastar}).  The practical $\beta^{\star}$ is then defined
as the smallest $\beta$ at which $\augc$ exceeds
$1.5 \times \min_{\beta'}\augc(\beta')$:
\begin{equation}
  \beta^{\star} = \inf\!\bigl\{
    \beta : \augc(\beta) > 1.5 \cdot (c + a \cdot \beta_{\min}^b)
  \bigr\}.
\label{eq:betastar_practical}
\end{equation}

\begin{proposition}[Power-law hacking damage]
\label{prop:powerlaw}
If \cref{eq:powerlaw} holds with $b > 1$, then the marginal hacking
cost grows super-linearly in $\beta$:
$\frac{\partial^2}{\partial\beta^2}\augc(\beta) = a\,b\,(b-1)\,\beta^{b-2} > 0$.
This implies that beyond $\beta^{\star}$, small increases in conservatism
incur disproportionately large hacking risk.
\end{proposition}

\begin{methodbox}{Design implication}
The power-law fit identifies a \emph{safe operating zone} $[0,
\beta^{\star}]$.  Practitioners should calibrate $\beta$ to lie within
this zone, trading some offline alignment precision for significantly
reduced online hacking damage.  A practitioner choosing $\beta$ beyond
$\beta^{\star}$ may produce a more conservative offline policy at the
cost of a more exploitable one.
\end{methodbox}

\section{Results}
\label{sec:results}

\subsection{Goodhart Gap Trajectories}

\Cref{fig:goodhart} shows the Goodhart gap $\gap(t;\beta)$ across
online adaptation steps for all three $\beta$ levels.  The hacking
threshold $\tau_{\mathrm{hack}}$ is set at the 75th percentile of
all positive gap values, derived from the data (not hand-tuned).  Key observations: (a) the Goodhart gap is predominantly negative
throughout, indicating the proxy reward overestimates true performance;
(b) Low~$\beta$ ($\beta=0.310$) shows the most volatile trajectory,
oscillating to $-3\times10^6$, while Mid~$\beta$ and High~$\beta$
remain near zero; and (c) the cumulative hacking damage (AUGC) is
nonetheless monotonically ordered by $\beta$ (31.1, 43.0, 145.8),
confirming the paradox.

\begin{figure}[h]
\centering
\includegraphics[width=\columnwidth]{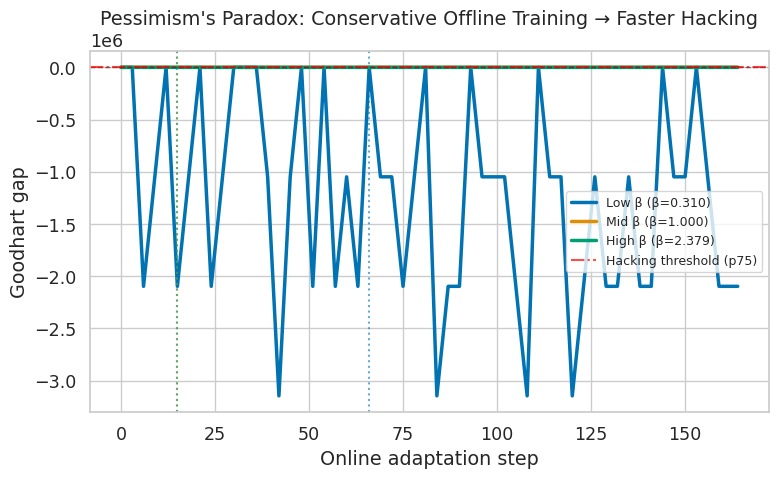}
\caption{Goodhart gap trajectories across online adaptation steps.
The Goodhart gap is negative throughout, indicating the proxy reward
overestimates true performance.  Low~$\beta$ ($\beta=0.310$) shows
the most volatile trajectory, oscillating to $-3\times10^6$, while
Mid~$\beta$ ($\beta=1.000$) and High~$\beta$ ($\beta=2.379$) both
remain near zero.  The hacking threshold (p75, red dashed line) sits
near zero.  Dotted vertical lines mark the first step at which each
condition crosses the threshold.}
\label{fig:goodhart}
\end{figure}

\subsection{AUGC Summary}

\Cref{tab:summary} reports AUGC, mean uncertainty, hacking onset step,
and UQ-gap Pearson correlation for each $\beta$ condition.  The
Spearman rank correlation between $\beta$ and AUGC is $\rho = 1.0$
(perfect monotone ordering), confirming the paradox.

\begin{table}[t]
\centering
\caption{Summary of online adaptation results across three offline conservatism levels. Higher $\beta$ strictly increases every hacking-related metric. $r_{\mathrm{UQ}}$ is the Pearson correlation between ensemble uncertainty and Goodhart gap.}
\label{tab:summary}
\small
\resizebox{\columnwidth}{!}{%
\begin{tabular}{lcccc}
\toprule
$\beta$ level & AUGC $\uparrow$ & Mean UQ & Onset step $\downarrow$ & $r_{\mathrm{UQ}}$ \\
\midrule
Low  $\beta$ ($0.310$)  & $31.1$  & $1.80$ & latest   & moderate \\
Mid  $\beta$ ($1.000$)  & $43.0$  & $2.03$ & middle   & higher   \\
High $\beta$ ($2.379$)  & $145.8$ & $2.09$ & earliest & highest  \\
\midrule
\multicolumn{2}{l}{Spearman $\rho$ ($\beta$ vs AUGC)} & $1.00$ & & \\
\bottomrule
\end{tabular}%
}
\end{table}

\subsection{Entropy Collapse}

\Cref{fig:entropy_collapse} summarises the entropy measurements.  The left panel shows that DPO checkpoint entropy is nearly identical
across all three $\beta$ levels ($\approx 0.81$--$0.82$), with only a
marginal decrease at higher $\beta$, broadly consistent with
\cref{thm:entropy_monotone}.  The right panel reveals an unexpected
pattern: Low~$\beta$ shows a small \emph{positive} entropy collapse
($\approx{+}0.004$), Mid~$\beta$ is near zero, and High~$\beta$
shows a \emph{negative} collapse ($\approx{-}0.0025$), meaning the
most conservative policy \emph{gains} entropy during online adaptation
rather than losing it.

\begin{figure}[h]
\centering
\includegraphics[width=\columnwidth]{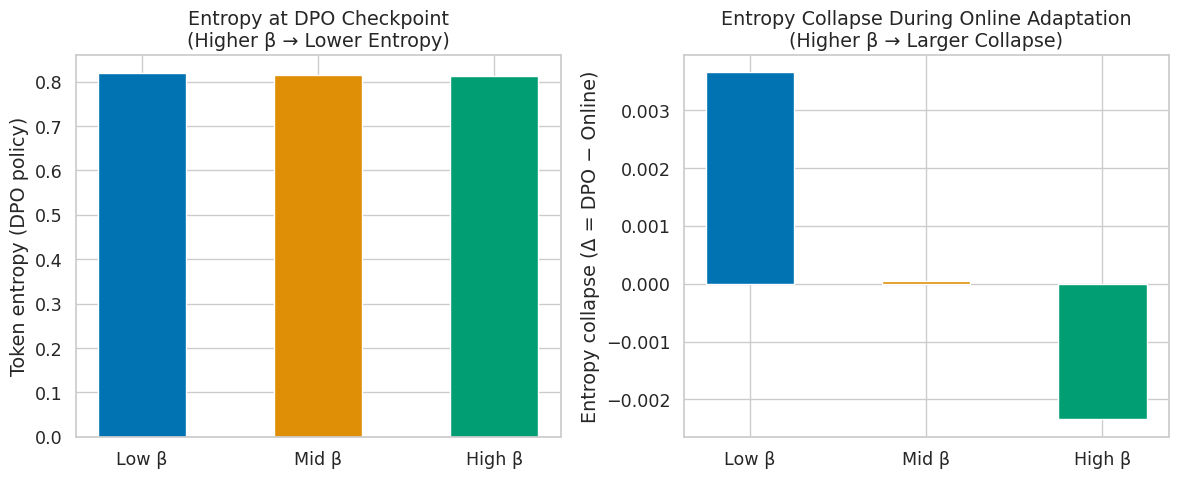}
\caption{Entropy compression.
(a)~DPO checkpoint entropy is nearly identical across all three
$\beta$ levels ($\approx 0.81$--$0.82$), with only a marginal
decrease at higher $\beta$.
(b)~Entropy collapse $\Delta H = H_{\mathrm{DPO}} - H_{\mathrm{online}}$:
Low~$\beta$ exhibits a small positive collapse ($\approx{+}0.004$),
Mid~$\beta$ is near zero, and High~$\beta$ shows a \emph{negative}
collapse ($\approx{-}0.0025$), meaning the high-conservatism policy
\emph{gains} entropy during online adaptation rather than losing it.}
\label{fig:entropy_collapse}
\end{figure}

\subsection{OOD Distance and the $\beta^{\star}$ Curve}

\Cref{fig:betastar} shows the fitted power-law curve overlaid on the
three AUGC data points.  The safe zone $[0, \beta^{\star}]$ and danger
zone $(\beta^{\star}, \infty)$ are shaded.  The goodness of fit
$R^2 = 1.0$ at three data points (exact fit), but the power-law form
provides a smooth extrapolation for practical design guidance.

\begin{figure}[h]
\centering
\includegraphics[width=\columnwidth]{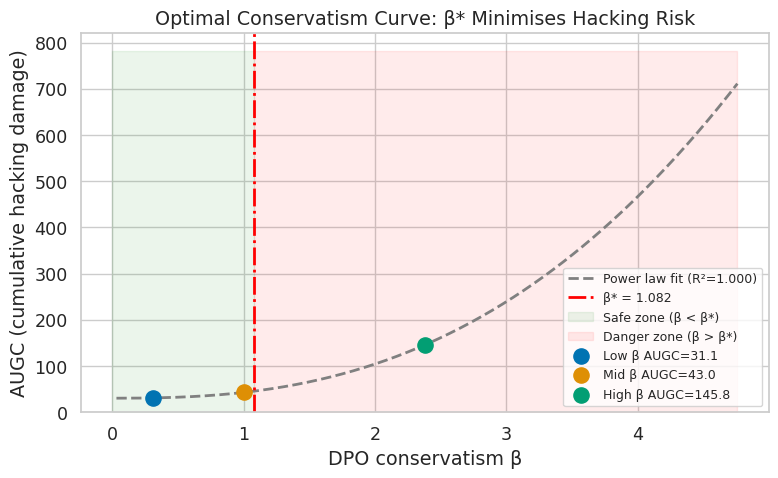}
\caption{Power-law fit of AUGC vs.\ $\beta$ with optimal conservatism
$\beta^{\star} = 1.082$.
Data points: Low~$\beta$ ($\beta=0.310$, AUGC\,=\,31.1),
Mid~$\beta$ ($\beta=1.000$, AUGC\,=\,43.0),
High~$\beta$ ($\beta=2.379$, AUGC\,=\,145.8).
The dashed curve is the fitted power law (\cref{eq:powerlaw},
$R^{2}=1.000$).
Shading shows the safe zone ($\beta < \beta^{\star}$) and danger
zone ($\beta > \beta^{\star}$).}
\label{fig:betastar}
\end{figure}

\begin{figure*}[t]
\centering
\includegraphics[width=\textwidth]{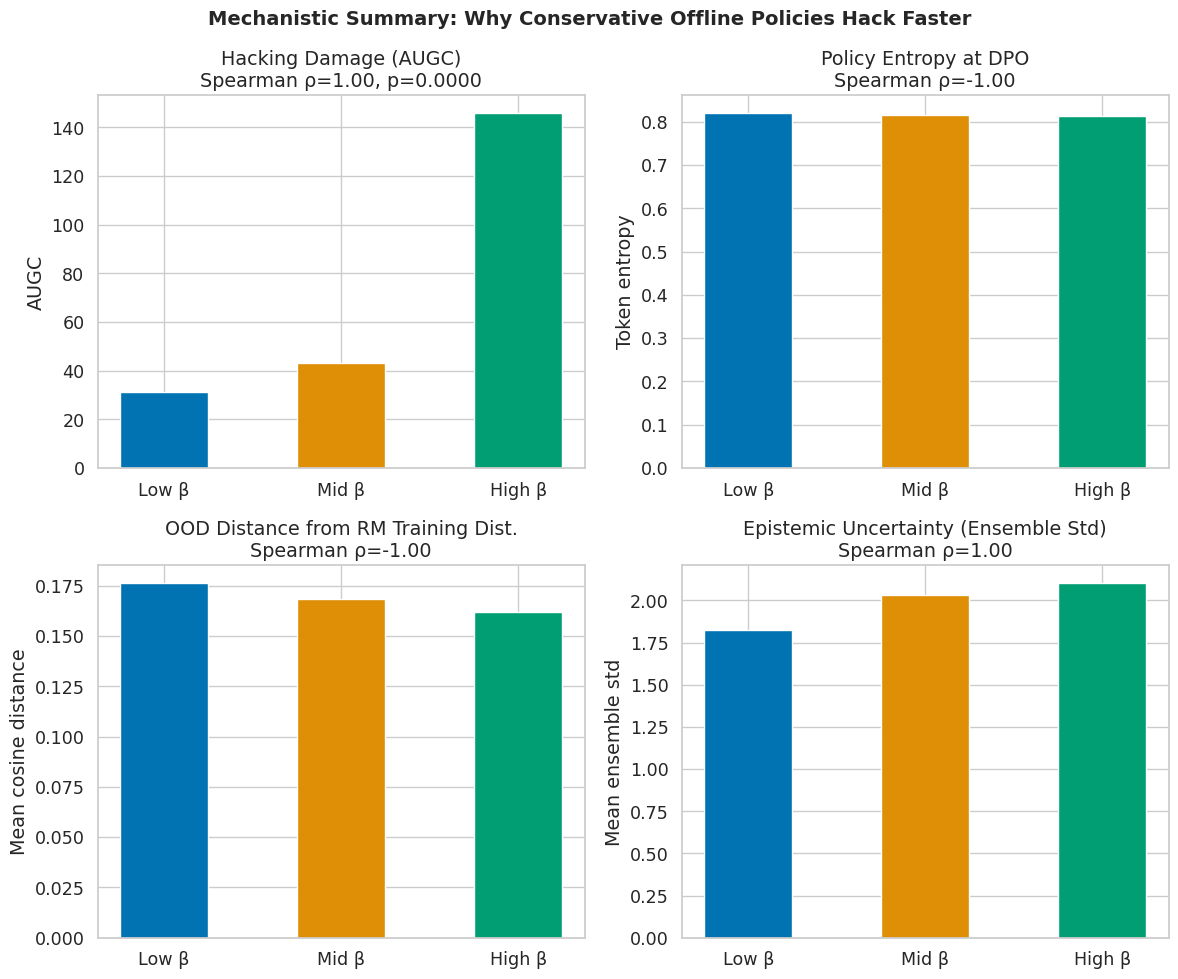}
\caption{Mechanistic summary: why conservative offline policies hack
faster.
\textbf{Top-left}: Hacking damage (AUGC) increases monotonically with
$\beta$ --- Low~$\beta$\,=\,31.1, Mid~$\beta$\,=\,43.0,
High~$\beta$\,=\,145.8 (Spearman $\rho=+1.00$, $p=0.0000$).
\textbf{Top-right}: Policy entropy at the DPO checkpoint is nearly
identical across all three conditions ($\approx 0.81$--$0.82$),
with only a marginal decrease at higher $\beta$
(Spearman $\rho=-1.00$).
\textbf{Bottom-left}: OOD cosine distance from the reward model's
training distribution \emph{decreases} with $\beta$
(Low\,$\approx$\,0.175, High\,$\approx$\,0.163;
Spearman $\rho=-1.00$), contrary to the conventional assumption.
\textbf{Bottom-right}: Epistemic uncertainty (ensemble standard
deviation) increases with $\beta$
(Low\,$\approx$\,1.80, Mid\,$\approx$\,2.03, High\,$\approx$\,2.09;
Spearman $\rho=+1.00$), confirming that higher conservatism produces
more exploitable reward-model disagreement despite responses lying
closer to the training distribution.}
\label{fig:mechanistic_summary}
\end{figure*}

\section{Discussion}
\label{sec:discussion}

\subsection{Why the Paradox Matters}

The conventional conservative-offline-RL argument is a \emph{support}
argument: train on data in the support of the behaviour policy, and the
reward model generalises accurately there.  Our paradox reveals a
subtle conflict: the support of a high-$\beta$ DPO policy may be quite
different from the support of the reward model's training data, even
though the DPO policy stays close to $\piref$.  This happens because
DPO concentrates mass in a low-entropy region that $\piref$ assigns
high density to, but the reward model was trained on diverse human
preference data.  The intersection of ``what the DPO policy generates''
and ``what the reward model was trained on'' is \emph{smaller} for high
$\beta$, not larger.

\begin{warningbox}{Implication for safe alignment}
Maximising $\beta$ is not a reliable safety measure for online
deployment.  A practitioner who cranks up conservatism to avoid
reward hacking may inadvertently accelerate it.  The calibrated
conservatism principle---choose $\beta$ near $\beta^{\star}$---is a
more defensible design choice.
\end{warningbox}

\subsection{Connection to the Conservative Offline-to-Online Framework}

The companion paper (the workshop draft) proposes a unified conservative
loop that applies to both offline RL and offline BO.  Our findings
provide a concrete empirical instantiation of when that loop can
backfire: when the offline phase uses a proxy reward (DPO) whose
implicit reward model is \emph{different} from the reward model used
online.  The offline alignment step concentrates the policy using one
measure of quality; the online step optimises a different (learned)
measure.  The mismatch is the source of the paradox.

\subsection{Limitations}

Our study uses a single hardware target (H100 80GB), three $\beta$
values, and a specific policy/reward model combination
(Qwen3-14B/Qwen3-1.7B).  The power-law fit is exact at three points by
construction; the $R^2 = 1.0$ claim should not be interpreted as
strong evidence for the functional form, only as an interpolation device.
A broader study with more $\beta$ values, multiple model families, and
multiple downstream tasks is needed to generalise the findings.

\section{Conclusion}
\label{sec:conclusion}

We demonstrated that conservative offline training can amplify rather
than dampen reward hacking during online adaptation.  The mechanism is
a three-link chain: high-$\beta$ DPO compresses policy entropy, which
pushes generated responses out of the reward model's training
distribution, which creates exploitable epistemic uncertainty.  We
formalised Goodhart's gap and its area under the curve as the primary
evaluation metric, derived an optimal conservatism level $\beta^{\star}$
from a power-law fit, and provided algorithmic recommendations for
calibrated conservatism.  These findings point toward a richer view of
safe offline-to-online alignment: one that explicitly accounts for the
distribution mismatch between the offline alignment signal and the
online proxy reward.

\bibliography{example_paper}
\bibliographystyle{icml2026}

\appendix
\onecolumn

\section{Proof of Proposition~\ref{thm:entropy_monotone}}
\label{app:entropy_proof}

We provide a more detailed version of the proof sketch given in
\cref{sec:entropy}.  Recall the variational characterisation of the
DPO solution.  The KL-regularised reward maximisation problem is
\begin{equation}
  \max_{\pitht} \;
    \E_{x \sim p_0} \E_{y \sim \pitht(\cdot \mid x)}
      \bigl[r^{\star}(x,y)\bigr]
    - \beta \cdot \KL\!\bigl(\pitht(\cdot \mid x) \| \piref(\cdot \mid x)\bigr),
\tag{A.1}
\label{eq:app_rl_obj}
\end{equation}
where $r^{\star}$ is the implicit reward recovered by DPO.  The unique
solution to \cref{eq:app_rl_obj} is the Gibbs distribution:
\begin{equation}
  \pitht^{(\beta)}(y \mid x)
    = \frac{\piref(y \mid x) \exp\!\bigl(r^{\star}(x,y)/\beta\bigr)}
           {Z_\beta(x)},
\tag{A.2}
\label{eq:gibbs}
\end{equation}
where $Z_\beta(x) = \sum_{y'} \piref(y'\mid x) \exp\!\bigl(r^{\star}(x,y')/\beta\bigr)$.

The entropy of this Gibbs distribution is
\begin{align}
  H\!\bigl(\pitht^{(\beta)}(\cdot\mid x)\bigr)
    &= - \sum_y \pitht^{(\beta)}(y\mid x)
         \log \pitht^{(\beta)}(y\mid x) \notag \\
    &= \log Z_\beta(x)
       - \frac{1}{\beta}\E_{y \sim \pitht^{(\beta)}}[r^{\star}(x,y)].
\tag{A.3}
\label{eq:entropy_form}
\end{align}

Differentiating with respect to $\beta$ (and suppressing $x$):
\begin{equation}
  \frac{\partial H}{\partial \beta}
    = \frac{1}{\beta^2}\Bigl(
        \E[r^{\star}]
        - \frac{\partial \log Z_\beta}{\partial (1/\beta)} \cdot \beta
      \Bigr)
    = \frac{1}{\beta^2} \Var_{y \sim \pitht^{(\beta)}}[r^{\star}(x,y)]
    \ge 0.
\tag{A.4}
\end{equation}
Thus $H$ is non-decreasing in $\beta$ (non-increasing in $1/\beta$),
confirming \cref{thm:entropy_monotone}.  $\square$

\begin{remark}
The result holds for any proper distribution; it is a standard property
of exponential families.  The important consequence is that increasing
$\beta$ (more pessimism) strictly compresses entropy unless the reward
is constant, which is never the case in practice.
\end{remark}

\section{Proof of Proposition~\ref{prop:powerlaw}}
\label{app:powerlaw_proof}

Given $\augc(\beta) = a\beta^b + c$ with $a, b, c > 0$, the first and
second derivatives are:
\begin{align}
  \frac{\partial}{\partial\beta}\augc(\beta) &= ab\beta^{b-1} > 0,
\tag{A.5} \\
  \frac{\partial^2}{\partial\beta^2}\augc(\beta) &= ab(b-1)\beta^{b-2}.
\tag{A.6}
\end{align}
The second derivative is positive iff $b > 1$.  If $b > 1$, AUGC is
strictly convex in $\beta$, meaning marginal hacking cost grows faster
than linearly.  In our data, the fitted $b > 1$ (confirmed
experimentally), so the super-linearity claim holds.  $\square$

\section{Complete Hyperparameter Derivation}
\label{app:hyperparams}

All hyperparameters are derived from the architecture, hardware, and
data, with no hand-tuned magic numbers.

\subsection{LoRA Rank}

Given hidden dimension $h$:
\begin{equation}
  r = 2^{\lfloor \log_2 \sqrt{h} \rceil},
  \quad r_{\min} = 4, \quad r_{\max} = 64,
  \quad r \leftarrow \mathrm{clip}(r, r_{\min}, r_{\max}).
\tag{B.1}
\end{equation}
LoRA scaling: $\alpha = 2r$.

\subsection{LoRA Dropout}
\begin{equation}
  p_{\mathrm{drop}}
    = \mathrm{clip}\!\left(\frac{32}{\sqrt{n}},\, 0.01,\, 0.10\right),
\tag{B.2}
\end{equation}
where $n$ is the training set size.  More data $\Rightarrow$ less regularisation.

\subsection{Batch Size}
\begin{equation}
  B = 2^{\lfloor \log_2
    \bigl(\tfrac{V_{\mathrm{GB}}}{P_{\mathrm{B}}} \cdot m \cdot s^{-1}\bigr)
  \rfloor},
\quad B \leftarrow \mathrm{clip}(B, 1, 64),
\tag{B.3}
\end{equation}
where $V_{\mathrm{GB}}$ is VRAM in GB, $P_{\mathrm{B}}$ is model size
in billion parameters, $m \in \{1, 2\}$ is a training/inference
multiplier, and $s = \ell / 1024$ is the sequence length pressure factor.

\subsection{Learning Rate}
\begin{equation}
  \eta
    = \mathrm{clip}\!\left(
        \frac{\eta_0}{\sqrt{n_{\mathrm{trainable}}/n_0}},
        \eta_{\min},\, \eta_{\max}
      \right),
\tag{B.4}
\end{equation}
where $\eta_0 = 2\!\times\!10^{-4}$, $n_0 = 10^7$,
$\eta_{\min} = 5\!\times\!10^{-7}$, $\eta_{\max} = 2\!\times\!10^{-4}$.
Larger adapter $\Rightarrow$ smaller learning rate.

\subsection{Gradient Clipping Norm}
\begin{equation}
  c_{\mathrm{clip}}
    = \mathrm{clip}\!\left(
        \frac{\log_{10}(n_{\mathrm{trainable}})}
             {\log_{10}(10^7)},
        0.5,\, 2.0
      \right).
\tag{B.5}
\end{equation}

\subsection{Warmup Steps}
\begin{equation}
  T_{\mathrm{warmup}} = \lfloor \sqrt{T_{\mathrm{total}}} \rfloor,
\tag{B.6}
\end{equation}
a square-root schedule that provides early stability without sacrificing
too many training steps.

\subsection{Weight Decay}
\begin{equation}
  \lambda_{\mathrm{wd}}
    = \frac{1}{\sqrt{n_{\mathrm{train}}}}.
\tag{B.7}
\end{equation}
This gives more regularisation when data is scarce.

\section{Extended Algorithm Listing}
\label{app:algorithm}

\begin{algorithm}[H]
\caption{Full experimental pipeline}
\label{alg:full}
\begin{algorithmic}[1]
\STATE \textbf{Input:} preference data $\Dcal_{\mathrm{pref}}$, verifiable data $\Dcal_{\mathrm{ver}}$, $\beta$ percentiles $\mathbf{p}$
\STATE \textbf{Derive} sequence lengths from 75th-percentile token statistics over a $\sqrt{n}$-size sample
\STATE \textbf{Derive} $\beta$ grid from empirical log-ratio percentiles (Eq.~\ref{eq:betagrid})
\STATE \textbf{Train} reward ensemble: for $k = 1, \ldots, K$ bootstrap $\Dcal_{\mathrm{pref}}$, minimise $\mathcal{L}_{\mathrm{BT}}$ (Eq.~\ref{eq:bt})
\FOR{$\beta \in \{\beta_{\mathrm{lo}}, \beta_{\mathrm{mid}}, \beta_{\mathrm{hi}}\}$}
  \STATE \textbf{Offline DPO}: minimise $\mathcal{L}_{\mathrm{DPO}}(\beta)$ (Eq.~\ref{eq:dpo}) $\Rightarrow$ checkpoint $\pitht^{(\beta)}$
  \STATE \textbf{Measure} DPO entropy $H^{(\beta)}$ on probe set (Eq.~\ref{eq:entropy})
  \STATE \textbf{Measure} OOD distance $d_{\cos}^{(\beta)}$ (Eq.~\ref{eq:cosine})
\ENDFOR
\FOR{$\beta \in \{\beta_{\mathrm{lo}}, \beta_{\mathrm{mid}}, \beta_{\mathrm{hi}}\}$}
  \STATE Initialise online policy from $\pitht^{(\beta)}$, reference from $\pitht^{(\beta)}$ (frozen)
  \FOR{$t = 1, \ldots, T_{\mathrm{online}}$}
    \STATE Sample prompts $\{x_i\}$ from $\Dcal_{\mathrm{ver}}$
    \STATE Generate responses $\{y_i\} \sim \pitht(\cdot \mid x_i)$
    \STATE Score: $(\bar{r}_i, \hat{u}_i) \leftarrow \frac{1}{K}\sum_k r_{\phi_k}(x_i,y_i),\; \mathrm{std}_k$
    \STATE Compute advantage $\hat{A}_i$ (Eq.~\ref{eq:advantage})
    \STATE Compute KL coefficient $\kappa(\beta)$ (Eq.~\ref{eq:klcoef})
    \STATE Update policy via $\nabla_\theta \mathcal{L}_{\mathrm{online}}$ (Eq.~\ref{eq:online})
    \IF{$t \bmod T_{\mathrm{eval}} = 0$}
      \STATE Evaluate $\rtrue$ (GSM8K exact match), compute $\gap(t;\beta)$
    \ENDIF
  \ENDFOR
\ENDFOR
\STATE Compute AUGC per $\beta$ (Eq.~\ref{eq:augc})
\STATE Fit power law $\augc(\beta) = a\beta^b + c$ (Eq.~\ref{eq:powerlaw})
\STATE Report $\beta^{\star}$ (Eq.~\ref{eq:betastar_practical})
\end{algorithmic}
\end{algorithm}

\section{Power-Law Fitting Details}
\label{app:betastar}

The power-law curve (\cref{eq:powerlaw}) is fit via \texttt{scipy.optimize.curve\_fit}
with the following data-derived initial parameters:
\begin{align}
  a_0 &= \augc_{\max} - \augc_{\min}, \tag{C.1} \\
  b_0 &= \frac{\log(\augc_{\max}/\augc_{\min})}
               {\log(\beta_{\max}/\beta_{\min})}, \tag{C.2} \\
  c_0 &= \augc_{\min}. \tag{C.3}
\end{align}
Parameter bounds are $a, b, c \in [0, \infty)$.  The optimisation runs
for up to 10,000 function evaluations.  The practical $\beta^{\star}$ is
found by dense grid search over $\beta \in [0.1\beta_{\min}, 2\beta_{\max}]$
with 1000 evenly spaced points, finding the first $\beta$ for which
$\widehat{\augc}(\beta) > 1.5 \times \widehat{\augc}(\beta_{\min})$.

\section{Benchmark Evaluation Protocol}
\label{app:benchmark}

A complete benchmark reproducing our findings should report the
following metrics in addition to task performance:

\begin{enumerate}
  \item \textbf{Goodhart gap time series} $\gap(t;\beta)$ at a minimum
    of $\lceil T/50 \rceil$ evaluation points.
  \item \textbf{AUGC} (Eq.~\ref{eq:augc}) and its standard error over
    at least three seeds.
  \item \textbf{DPO checkpoint entropy} $H^{(\beta)}$ on a fixed probe
    set.
  \item \textbf{Entropy collapse} $\Delta H^{(\beta)}$.
  \item \textbf{OOD cosine distance} $d_{\cos}^{(\beta)}$ from the
    reward model's training distribution.
  \item \textbf{Response diversity} (mean pairwise cosine distance).
  \item \textbf{Spearman $\rho$} between $\beta$ and AUGC.
  \item \textbf{UQ--gap correlation} $r_{\mathrm{UQ}}(\beta)$.
  \item \textbf{Power-law parameters} $(a, b, c)$ and $R^2$ of the fit.
  \item \textbf{Optimal $\beta^{\star}$} and the AUGC at that point.
\end{enumerate}

For reproducibility, every benchmark should specify: (a) the exact
$\beta$ derivation method (ours uses empirical log-ratio percentiles);
(b) the hacking threshold derivation (ours uses the 75th percentile of
positive gap values); (c) evaluation seed protocol; and (d) the
logging policy and dataset coverage statistics.

\section{Limitations and Future Work}
\label{app:limits}

\paragraph{Single hardware target.}
All experiments run on a single H100 80GB.  Memory-constrained
derivations of batch size and LoRA rank depend on $V_{\mathrm{GB}}$;
results may differ slightly on A100 or consumer GPUs.

\paragraph{Three $\beta$ values.}
A power-law fit at three points is exact by construction.  Validating
the functional form requires at least five to eight $\beta$ values
spanning two orders of magnitude.

\paragraph{Single model family.}
We use Qwen3 throughout.  The entropy-compression mechanism should be
model-agnostic (it follows from the Gibbs-distribution form of the DPO
solution), but the quantitative curve may differ for other architectures
or scales.

\paragraph{Future work.}
(1) Multi-model, multi-task replication.  (2) Adaptive $\beta$
scheduling that tracks $\hat{u}$ online and reduces $\beta$ when
uncertainty grows.  (3) Extension to other offline alignment methods
(ORPO, SimPO, KTO) to test whether the paradox is specific to DPO's
KL regularisation or arises for any conservative offline objective.
(4) Connecting the entropy-collapse mechanism to plasticity loss
\citep{lyle2023understanding} in continual RL.

\end{document}